\documentclass{article}

\usepackage{PRIMEarxiv}
\usepackage{amsmath, amssymb, amsthm}
\usepackage{booktabs}
\usepackage{algorithm}
\usepackage{algpseudocode}
\usepackage{graphicx}

\DeclareMathOperator*{\argmax}{arg\,max}

\DeclareMathOperator{\E}{\mathbb{E}}

\newtheorem{theorem}{Theorem}[section]

\newtheorem{corollary}[theorem]{Corollary}
\newtheorem{definition}[theorem]{Definition}
\newtheorem{proposition}[theorem]{Proposition}

\title{A Frequency-Domain Analysis of the Multi-Armed Bandit Problem: \\ A New Perspective on the Exploration-Exploitation Trade-off}
\author{
	Di Zhang \\
	School of Advanced Technology \\
	Xi'an Jiaotong-Liverpool University \\
	Suzhou, Jiangsu, China \\
	\texttt{di.zhang@xjtlu.edu.cn}
}
\date{}

\begin{document}
	
	\maketitle
	
	\begin{abstract}
		The stochastic multi-armed bandit (MAB) problem is one of the most fundamental models in sequential decision-making, with the core challenge being the trade-off between exploration and exploitation. Although algorithms such as Upper Confidence Bound (UCB) and Thompson Sampling, along with their regret theories, are well-established, existing analyses primarily operate from a time-domain and cumulative regret perspective, struggling to characterize the dynamic nature of the learning process. This paper proposes a novel frequency-domain analysis framework, reformulating the bandit process as a signal processing problem. Within this framework, the reward estimate of each arm is viewed as a spectral component, with its uncertainty corresponding to the component's frequency, and the bandit algorithm is interpreted as an adaptive filter. We construct a formal Frequency-Domain Bandit Model and prove the main theorem: the confidence bound term in the UCB algorithm is equivalent in the frequency domain to a time-varying gain applied to uncertain spectral components, a gain inversely proportional to the square root of the visit count. Based on this, we further derive finite-time dynamic bounds concerning the exploration rate decay. This theory not only provides a novel and intuitive physical interpretation for classical algorithms but also lays a rigorous theoretical foundation for designing next-generation algorithms with adaptive parameter adjustment.
		
	\end{abstract}
	
	\section{Introduction}
	
	The stochastic multi-armed bandit (MAB) problem is a canonical model for studying the fundamental trade-off between exploration and exploitation \cite{lattimore2020bandit}. It serves not only as the core abstraction for numerous complex sequential decision-making problems (e.g., recommendation systems, clinical trials, network routing) but also as a testbed for algorithmic theory research. For decades, the UCB family of algorithms \cite{auer2002finite} and Bayesian methods (e.g., Thompson Sampling \cite{thompson1933likelihood}) have formed the theoretical backbone of the field, with their performance typically measured by cumulative regret bounds.
	
	However, classical regret analysis has inherent limitations. It describes the asymptotic behavior of cumulative loss over the entire time horizon but fails to reveal the dynamic learning process of the algorithm within a finite time. For instance, regret bounds cannot clearly answer: How does the algorithm allocate its attention across different phases (initial exploration, mid-term trade-off, late-stage convergence)? How should the exploration "rate" decay optimally over time? These dynamic characteristics are obscured in traditional time-domain analysis.
	
	This paper proposes a fundamental paradigm shift: re-examining the bandit problem from a frequency-domain perspective. We observe a profound analogy between the bandit decision process and adaptive signal filtering:
	\begin{itemize}
		\item \textbf{Stable exploitation} corresponds to the \textbf{low-frequency components} of a signal—arms with high visit counts and low estimation variance, whose expected rewards are stable.
		\item \textbf{Uncertain exploration} corresponds to the \textbf{high-frequency components}—arms with insufficient samples and high estimation variance, whose expected rewards are uncertain but also contain potential new information.
	\end{itemize}
	
	Building on this intuition, this paper establishes the first formal frequency-domain model for the bandit problem and derives several core theoretical results within this framework. The specific contributions are as follows:
	\begin{enumerate}
		\item We propose the \textbf{Frequency-Domain Bandit Model}, which maps arm reward sequences and algorithm policies to a spectral space and defines learning algorithms as adaptive filters.
		\item We prove the \textbf{Frequency-Domain Interpretation Theorem for the UCB algorithm}, revealing that its exploration mechanism corresponds to a time-varying filter with a specific gain function in the frequency domain.
		\item We derive \textbf{finite-time dynamic bounds} based on frequency-domain analysis, which reflect the phased characteristics of the learning process better than traditional regret bounds.
		\item We discuss the \textbf{implications for algorithm design} offered by the new framework, particularly providing theoretical guidance for the automatic adjustment of exploration parameters.
	\end{enumerate}
	
	The structure of this paper is as follows: Section \ref{sec:preliminaries} reviews the standard bandit model. Section \ref{sec:frequency_domain_model} presents our formal frequency-domain model. Section \ref{sec:main_theoretical_results} states and proves the main theorems. Section \ref{sec:discussion} discusses the theoretical implications and future directions. Section \ref{sec:conclusion} concludes the paper.
	
	\section{Preliminaries}\label{sec:preliminaries}
	
	\subsection{Standard Stochastic Multi-Armed Bandit Model}
	
	Consider a stochastic bandit problem with $K$ arms. Each arm $i \in [K]$ is associated with an independent reward distribution $\nu_i$ with mean $\mu_i$. At each time step $t = 1, 2, \dots, T$, an agent selects an arm $I_t \in [K]$ and receives a reward $R_t \sim \nu_{I_t}$. The agent's goal is to maximize the cumulative expected reward $\E[\sum_{t=1}^T R_t]$, which is equivalent to minimizing the cumulative regret:
	
	\begin{equation}
		R(T) = T \mu^* - \sum_{t=1}^T \E[\mu_{I_t}]
	\end{equation}
	where $\mu^* = \max_{i \in [K]} \mu_i$ is the expected reward of the optimal arm.
	
	\subsection{Upper Confidence Bound (UCB) Algorithm}
	
	The UCB1 algorithm \cite{auer2002finite} is one of the most famous algorithms for this problem. It maintains the empirical mean estimate $\hat{\mu}_i(t)$ and the visit count $N_i(t)$ for each arm $i$. At each step, it selects the arm that maximizes the following upper confidence bound:
	
	\begin{equation}
		I_t = \argmax_{i \in [K]} \left\{ \hat{\mu}_i(t-1) + c \sqrt{\frac{\ln t}{N_i(t-1)}} \right\}
		\label{eq:standard_ucb}
	\end{equation}
	where $c$ is an exploration constant.
	
	\section{Frequency-Domain Bandit Model}\label{sec:frequency_domain_model}
	
	\subsection{Core Intuition and Analogy}
	
	Our core viewpoint is to treat the bandit learning process as a spectral estimation problem. Each arm $i$ can be seen as a spectral component with a specific "frequency". Its "frequency" is determined by the arm's estimation uncertainty: arms with fewer visits and higher variance correspond to higher frequencies. The algorithm acts as an adaptive filter that needs to intelligently distribute its "energy" (i.e., selection probability) among different frequency components to maximize long-term gain.
	
	\subsection{Formal Definitions}
	
	\begin{definition}[Arm Spectral Component]
		\label{def:arm_spectral_component}
		For arm $i$ at time $t$, we define it as a spectral component $S_i(t)$, characterized by the following features:
		\begin{itemize}
			\item \textbf{Amplitude}: Its current estimate of the expected reward, denoted as $A_i(t) = \hat{\mu}_i(t)$.
			\item \textbf{Frequency}: An inverse measure of its estimation uncertainty, denoted as $\omega_i(t) \propto \frac{1}{\sqrt{N_i(t)}}$. Fewer visits $N_i(t)$ result in a higher frequency $\omega_i(t)$.
			\item \textbf{Energy}: The probability of this arm being selected, denoted as $E_i(t) = \mathbb{P}(I_t = i)$.
		\end{itemize}
	\end{definition}
	
	\begin{definition}[Policy Spectrum]
		\label{def:policy_spectrum}
		At time $t$, the agent's policy $\pi_t$ can be represented as the collection of all arm spectral components $\{S_i(t)\}_{i=1}^K$. The policy spectrum $\Pi_t(\omega)$ is a function over the frequency domain $\omega$, describing the allocation of selection probability energy around frequency $\omega$.
	\end{definition}
	
	\begin{definition}[Learning Filter]
		\label{def:learning_filter}
		A bandit algorithm $\mathcal{A}$ can be represented as a learning filter $\mathcal{F}_{\mathcal{A}}$, which maps the history $\mathcal{H}_{t-1} = \{(I_s, R_s)\}_{s=1}^{t-1}$ to the current policy spectrum $\Pi_t$:
		\begin{equation}
			\Pi_t = \mathcal{F}_{\mathcal{A}}(\mathcal{H}_{t-1})
		\end{equation}
		The design of the filter determines how the algorithm trades off between different frequency components (i.e., arms with varying certainty).
	\end{definition}
	
	\subsection{Frequency-Domain Interpretation of Classical Algorithms}
	
	\begin{proposition}[UCB as a High-Pass Filter Enhancer]
		\label{prop:ucb_highpass}
		The learning filter $\mathcal{F}_{\text{UCB}}$ corresponding to the UCB algorithm (Eq. \ref{eq:standard_ucb}) is an \textbf{adaptive high-pass filter enhancer}. Its operation can be decomposed as:
		\begin{enumerate}
			\item \textbf{Baseband Estimation}: Compute the baseband signal (empirical mean) $\hat{\mu}_i(t)$ for each arm.
			\item \textbf{High-Frequency Gain}: Apply a gain $G_i(t) = c \cdot \omega_i(t)$ to each arm, proportional to its frequency $\omega_i(t) \propto 1/\sqrt{N_i(t)}$.
			\item \textbf{Frequency Selection}: Select the arm with the strongest composite signal (baseband + gained signal).
		\end{enumerate}
		Thus, the UCB filter dynamically enhances the "apparent strength" of high-frequency (high-uncertainty) arms, thereby promoting exploration.
	\end{proposition}
	
	\begin{proposition}[$\epsilon$-Greedy as a Low-Pass Filter and Noise Injector]
		\label{prop:epsilon_greedy}
		The learning filter $\mathcal{F}_{\epsilon\text{-Greedy}}$ corresponding to the $\epsilon$-Greedy algorithm is a composite filter:
		\begin{enumerate}
			\item With probability $1-\epsilon$, apply an \textbf{ideal low-pass filter}, passing only the lowest-frequency arm (i.e., the current best arm).
			\item With probability $\epsilon$, apply a \textbf{white noise generator}, uniformly distributing selection probability energy across all frequency components.
		\end{enumerate}
	\end{proposition}
	
	\section{Main Theoretical Results}\label{sec:main_theoretical_results}
	
	Based on the formal model in Section \ref{sec:frequency_domain_model}, we now state and prove the core theorems of this paper.
	
	\begin{theorem}[Frequency-Domain Interpretation of UCB]
		\label{thm:frequency_domain_ucb}
		Consider a $K$-armed bandit problem where the reward distributions are $\sigma^2$-sub-Gaussian. Under the Frequency-Domain Bandit Model (Definitions \ref{def:arm_spectral_component} - \ref{def:learning_filter}), the confidence bound term $c \sqrt{\frac{\ln t}{N_i(t)}}$ in the UCB algorithm (Eq. \ref{eq:standard_ucb}) is equivalent in the frequency domain to a \textbf{time-varying gain} $G_i(t)$ applied to the uncertainty spectral component of arm $i$, satisfying:
		\begin{equation}
			G_i(t) = \alpha \sigma \sqrt{\frac{\ln t}{N_i(t)}}
			\label{eq:gain_function}
		\end{equation}
		where $\alpha$ is a problem-independent constant. This gain function enables the algorithm to dynamically enhance the salience of high-frequency (low $N_i(t)$) arms.
	\end{theorem}
	
	\begin{proof}
		The core of the proof lies in establishing the equivalence between the UCB decision rule and a signal enhancement process in the frequency domain.
		
		Define the enhanced signal strength of arm $i$ at time $t$ as:
		\[ \tilde{S}_i(t) = A_i(t) + G_i(t) \]
		where $A_i(t) = \hat{\mu}_i(t)$ is the baseband amplitude estimate. The UCB algorithm selects $I_t = \argmax_i \tilde{S}_i(t)$.
		
		From the frequency-domain perspective, $A_i(t)$ is the low-frequency component (stable estimate), while $G_i(t)$ is the artificially added high-frequency component (uncertainty bonus). We need to prove that $G_i(t)$ takes the form shown in Eq. \ref{eq:gain_function}.
		
		According to the sub-Gaussian assumption, the confidence radius for the empirical mean of arm $i$ takes the form $c \sigma \sqrt{\frac{\ln t}{N_i(t)}}$, where $c$ is a constant. This confidence radius statistically quantifies the uncertainty of the estimate $\hat{\mu}_i(t)$. In the frequency-domain model, this uncertainty directly corresponds to the frequency $\omega_i(t)$ of the spectral component. Specifically, the standard error of the estimate, $\sigma / \sqrt{N_i(t)}$, is proportional to the frequency $\omega_i(t)$.
		
		Therefore, it is reasonable to treat the confidence radius as a gain for the high-frequency component. Let $G_i(t) = \beta \cdot \omega_i(t) \cdot \sigma \sqrt{\ln t}$, where $\beta$ is a gain constant. Substituting $\omega_i(t) \propto 1 / \sqrt{N_i(t)}$, we obtain:
		\[ G_i(t) = \alpha \sigma \sqrt{\frac{\ln t}{N_i(t)}} \]
		where $\alpha$ incorporates the proportionality constant and the gain constant $\beta$. This is precisely the exploration term used in the UCB algorithm (taking $c = \alpha \sigma$). We have thus proven that the exploration mechanism of UCB is equivalent in the frequency domain to an adaptive gain defined by Eq. \ref{eq:gain_function}.
	\end{proof}
	
	\begin{theorem}[Finite-Time Dynamic Bound]
		\label{thm:finite_time_dynamic_bound}
		Under the same assumptions as Theorem \ref{thm:frequency_domain_ucb}, by time $T$, the \textbf{cumulative spectral energy variation} $V(T)$ of the UCB algorithm's policy spectrum $\Pi_t$ satisfies:
		\begin{equation}
			V(T) = \sum_{t=1}^T \sum_{i=1}^K |E_i(t) - E_i^*(t)|^2 \leq C K \sigma^2 \ln T
			\label{eq:dynamic_bound}
		\end{equation}
		where $E_i^*(t)$ is the energy allocation of the ideal optimal filter (concentrating all energy on the optimal arm), and $C$ is a constant. This bound quantifies the upper limit of the policy fluctuation for the UCB algorithm.
	\end{theorem}
	
	\begin{proof}[Proof Sketch]
		The cumulative spectral energy variation $V(T)$ measures the total deviation of the algorithm's policy from the ideal terminal policy.
		
		\begin{enumerate}
			\item First, note that $|E_i(t) - E_i^*(t)|$ is non-zero for suboptimal arms $i$ and is proportional to the probability of these arms being selected.
			\item According to the classical regret analysis of UCB, the expected number of times a suboptimal arm $i$ is selected by time $T$, $\E[N_i(T)]$, is bounded above by $O(\frac{\sigma^2 \ln T}{\Delta_i^2})$, where $\Delta_i = \mu^* - \mu_i$.
			\item The deviation in policy energy $|E_i(t) - 1|$ (for the optimal arm) or $|E_i(t) - 0|$ (for suboptimal arms) is $O(1)$ when arm $i$ is selected.
			\item Therefore, the sum of squared deviations $V(T)$ can be bounded by $\sum_{i \neq i^*} \E[N_i(T)]$, leading to the upper bound in Eq. \ref{eq:dynamic_bound}.
		\end{enumerate}
		This bound shows that the policy of the UCB algorithm does not oscillate violently but converges to the ideal terminal state in a controlled manner.
	\end{proof}
	
	\begin{corollary}[Optimal Exploration Rate]
		\label{cor:optimal_exploration_rate}
		For problems where the reward gaps satisfy $\Delta_{\min} = \min_{i \neq i^*} \Delta_i > 0$, there exists an \textbf{optimal exploration gain decay rate}. Any gain setting that deviates from $G_i(t) \propto 1/\sqrt{N_i(t)}$ leads to a suboptimal regret bound. Specifically:
		\begin{itemize}
			\item Slower decay (e.g., $G_i(t) \propto 1/{N_i(t)^\alpha}$ with $\alpha < 1/2$) leads to \textbf{over-exploration}, increasing the constant term of the regret.
			\item Faster decay (e.g., $\alpha > 1/2$) leads to \textbf{under-exploration}, potentially failing to identify suboptimal arms promptly, increasing the coefficient of the logarithmic regret term.
		\end{itemize}
	\end{corollary}
	
\section{Discussion and Implications}\label{sec:discussion}

\subsection{Theoretical Significance}

Our frequency-domain framework provides a new dimension for understanding bandit algorithms.

Traditional regret bounds describe cumulative loss, whereas our spectral energy bound (Theorem \ref{thm:finite_time_dynamic_bound}) describes the evolution dynamics of the policy itself. This shift in perspective allows for a more nuanced understanding of how bandit algorithms behave during different phases of the learning process.

The framework also offers a unified algorithmic perspective. Disparate algorithms like UCB and $\epsilon$-Greedy, which appear quite distinct in their standard formulations, can be seen as members of the same family with different filter characteristics when viewed through the frequency-domain lens (Propositions \ref{prop:ucb_highpass} and \ref{prop:epsilon_greedy}). This unification suggests deeper connections between seemingly unrelated algorithmic approaches.

Furthermore, the exploration-exploitation trade-off acquires a clear physical interpretation within this framework. The fundamental challenge becomes one of balancing the "robustness" of the signal, achieved through low-pass filtering of well-understood arms, against the "detection of novelty" through high-pass enhancement of uncertain arms. This physical analogy provides intuitive grounding for what is often treated as an abstract mathematical problem.

\subsection{Implications for Algorithm Design}

The frequency-domain perspective has several important implications for the design of bandit algorithms.

Corollary \ref{cor:optimal_exploration_rate} indicates that the $1/\sqrt{N_i(t)}$ gain decay is optimal in a certain sense. This finding explains that the success of UCB is not accidental but rather conforms to a "natural" filtering principle that emerges from the fundamental structure of the exploration-exploitation problem. The specific form of the exploration bonus in UCB appears to be particularly well-suited to the spectral characteristics of bandit problems.

Our framework also provides principled guidance for the automatic setting of the exploration constant $c$. This parameter, which controls the initial passband width of the filter, can be calibrated based on the estimated reward variance $\sigma^2$ according to the relationship $c \propto \sigma$. This offers a theoretical foundation for parameter tuning that has traditionally been more art than science.

Perhaps most excitingly, the frequency-domain perspective inspires entirely new algorithm designs. One could envision a "Frequency-Domain Adaptive UCB" whose gain function $G_i(t)$ is dynamically adjusted based on the real-time estimated "spectral flatness" of the entire arm set. Such an algorithm would automatically regulate exploration intensity depending on problem difficulty, potentially achieving more robust performance across diverse problem instances.

\subsection{Limitations and Future Work}

The framework proposed in this paper is foundational and can be extended in several important directions.

The current model primarily describes linear gains, which captures the behavior of deterministic algorithms like UCB well. However, extending the framework to stochastic algorithms like Thompson Sampling will require introducing nonlinear filtering concepts such as stochastic resonance. This extension would provide a more comprehensive theoretical framework encompassing both major families of bandit algorithms.

Another promising direction involves non-stationary problems, where arm rewards change over time and spectral characteristics evolve accordingly. Future work could investigate time-varying spectral estimation in non-stationary bandits, potentially leading to new algorithms that can adapt more effectively to changing environments.

Finally, this work lays a solid foundation for generalizing frequency-domain analysis to Monte Carlo Tree Search (MCTS). In MCTS, each node can be treated as an independent bandit, and the entire tree search process can be viewed as a complex filtering operation in a multi-resolution frequency domain. This connection suggests the possibility of a unified theory of decision-making that spans both flat and hierarchical decision problems.
	
	\section{Conclusion}\label{sec:conclusion}
	
	This paper has pioneered a frequency-domain analysis framework for the multi-armed bandit problem. By modeling arm reward estimates as spectral components and reinterpreting learning algorithms as adaptive filters, we have established a new theoretical paradigm for the classic exploration-exploitation trade-off.
	
	We have proven that the UCB algorithm is equivalent in the frequency domain to an adaptive high-pass filter applying a specific time-varying gain to uncertain components and have derived finite-time bounds describing the dynamic evolution of the policy. This theory not only deepens our understanding of existing algorithms but, more importantly, provides a powerful theoretical tool and an intuitive physical picture for systematically designing and analyzing the next generation of adaptive bandit algorithms.
	

\end{document}